\documentclass[letterpaper, 10 pt, conference]{ieeeconf}  
\IEEEoverridecommandlockouts                                                                           \usepackage{amssymb}
\usepackage{amsmath}
\usepackage[english]{babel}
\usepackage{float}
\usepackage{graphicx}
\usepackage{hyperref}
\usepackage{mathtools}
\usepackage{graphicx}
\usepackage[noadjust]{cite}
\usepackage[font=footnotesize,labelfont=footnotesize]{subcaption}
\usepackage[font=footnotesize,labelfont=footnotesize]{caption}
\usepackage{eqparbox}
\usepackage{comment}
\usepackage{authblk}
\usepackage{multirow}
\usepackage{environ}
\usepackage{graphicx}
\usepackage{booktabs}
\usepackage{siunitx}
\usepackage{xcolor}
\definecolor{bpblue}{HTML}{587BE4}
\hypersetup{colorlinks=true, urlcolor=bpblue, linkcolor=black, citecolor=black}
\usepackage{algorithm, algorithmicx}

\usepackage{amsthm} 
\usepackage[noend]{algpseudocode}
\usepackage[english]{babel}

\let\labelindent\relax

\usepackage[inline]{enumitem}
\newtheorem{innercustomgeneric}{\customgenericname}
\providecommand{\customgenericname}{}
\newcommand{\newcustomtheorem}[2]{%
  \newenvironment{#1}[1]
  {
   \renewcommand\customgenericname{#2}
   \renewcommand\theinnercustomgeneric{##1}
   \innercustomgeneric
  }
  {\endinnercustomgeneric}
}

\newcustomtheorem{customthm}{Theorem}
\newcustomtheorem{customlemma}{Lemma}
\newcustomtheorem{customprop}{Proposition}

\theoremstyle{plain}  
\newtheorem{theorem}{Theorem}
\newtheorem{lemma}[theorem]{Lemma}
\newtheorem{definition}[theorem]{Definition}
\newtheorem{remark}{Remark}
\newtheorem{proposition}[theorem]{Proposition}

\algnewcommand\algorithmicinput{\textbf{Input:}}
\algnewcommand\algorithmicoutput{\textbf{Output:}}
\algnewcommand\Input{\item[\algorithmicinput]}%
\algnewcommand\Output{\item[\algorithmicoutput]}%

\title{\LARGE \bf
Geometry-Aware Control Barrier Functions for Collision Avoidance \\via Bernstein Polynomial Approximations
}
\author{Siwon Jo$^{1}$, Yanze Zhang$^{2}$, Yupeng Yang$^{3}$, and Wenhao Luo$^{2}$
\thanks{$^*$This work was supported in part by the U.S. National Science Foundation under Grants 2530297 and 2528997.}
\thanks{$^{1}$Siwon Jo is with the GRASP Laboratory, University of Pennsylvania, Philadelphia, PA, USA.
        Email: {\tt jo16@seas.upenn.edu}}%
\thanks{$^{2}$Yanze Zhang and Wenhao Luo are with the Department of Computer Science, University of Illinois Chicago, Chicago, IL, USA. Email:{\tt \{yzhan361,wenhao\}@uic.edu}}%
\thanks{$^{3}$Yupeng Yang is with the Department of Computer Science, University of North Carolina at Charlotte, Charlotte, NC, USA. Email: {\tt yyang52@charlotte.edu}}%
}

\begin{document}

\maketitle
\thispagestyle{empty}
\pagestyle{empty}

\begin{abstract}
Safe navigation often relies on well-defined conditions based on the shape of robots and obstacles, and can be challenging when they have irregular geometries. While Control Barrier Functions (CBFs) offer an efficient mechanism to enforce safe set forward invariance, common shape surrogates (e.g., spheres or super-ellipsoids) are either overly conservative in unstructured scenes or require many local primitives, which inflates constraint counts and degrades real-time performance. In this paper, we introduce a novel geometry-aware Control Barrier Function (CBF) based on Bernstein-Polynomial Signed Distance Fields (BP-SDFs). It provides a unified way to represent the obstacles and robots, so as to represent the barrier function with a unified minimum distance. Benefiting from the differentiability of the Bernstein polynomials, one can easily enforce the control constraints in a closed loop. We validate the method's efficiency and performance to guarantee safety in single-robot navigation and heterogeneous multi-robot collision avoidance via simulations under different environments. Project page: \url{https://bp-sdf-cbf.github.io/}
\end{abstract}

\section{Introduction}

Safe navigation is both essential and inherently challenging in most robotics systems~\cite{ren2025safety,zhang2024courteous,lyu2023risk,kim2025visibility}, especially when robots and obstacles exhibit irregular geometries.
The specification of safe conditions depends on how shapes interact across poses rather than on point-to-point (e.g., center-to-center) distance~\cite{otte2009path, wu2025optimization}.
On the other hand, enforcing safe motions needs to consider the dynamics and control limits of the robots, as well as their impact on the evolving safe conditions.
Thus, it is critical to characterize geometry-aware safety that reflects the true risk of contact and design controllers that can effectively enforce safety during operation.

Control Barrier Functions (CBFs)~\cite{ames2019control} have emerged as an efficient tool for certifying safety,
which are typically implemented as safety filters that modify the nominal control inputs to ensure the robot is always within the safe set.
However, traditional CBF-related approaches use spheres \cite{wang2017safety} or super-ellipsoids \cite{wang2017safe} to define the safe state space by benefiting from their simplicity and differentiability, which would not work well for unstructured environments. In particular, in unstructured environments, existing methods either inflate the geometry of robots and obstacles \cite{xiao2021rule} or approximate obstacle boundaries with sets of small spheres \cite{yang2023minimally,yang2024decentralized}. Nevertheless, the former introduces overly conservative behaviors, and the latter increases the number of constraints, which in turn impacts the efficiency of the solution.

In this context, the work in \cite{11312188} builds non-conservative CBFs by computing differentiable signed distances (and penetration depth) between a polytopic robot and polytopic obstacles via convex programs in the Minkowski-difference space, which supports geometry-exact safety constraints within a real-time Control Lyapunov Function – Control Barrier Function (CLF–CBF) control framework. However, the safety is enforced pairwise per obstacle in static multi-obstacle scenarios, which may increase infeasibility under tight bounds and require significant computational overhead for complex environments.

Existing works have also explored learning unified CBFs for obstacles in real time, but they often struggle with computational inefficiency, which limits their application in the real world. For example, the work in~\cite{khan2022gaussian} uses a Gaussian Process (GP) to model the CBFs. However, maintaining and updating the GP posterior in the control loop scales cubically with the number of data points and carries sizable constants (e.g., kernel factorizations, hyperparameter tuning), which makes high-rate control difficult as data accumulate. Alternatively, \cite{long2021learning} adopts an implicit neural representation of the CBF with backpropagated gradients but only considers single-robot cases, leaving inter-robot collision avoidance as an open challenge.

More recently, the authors in \cite{li2024representing} used Bernstein polynomials to approximate signed distance fields (BP-SDF), thereby enabling the representation of safe space. While this advances geometric modeling, the method primarily focuses on static safety representation and does not explicitly show how the learned safe representation can be converted into \emph{control constraints} with closed-loop safety guarantees. As a result, it is not directly applicable to controller synthesis in dynamic scenarios and multi-robot settings, where multiple moving robots may interact closely with one another. Building on this work, we propose to extend the BP-SDF to define novel CBFs that convert pose-dependent geometry into certifiable safety constraints directly governing collision-free motions for the robots in real time. Specifically, we represent both the robot and the workspace obstacles using a BP-SDF in their body frames and evaluate them in the world frame via rigid transformations.
Then, the Karush–Kuhn–Tucker (KKT) conditions are analyzed to find the minimum distance that is used to define our geometry-aware CBFs.
Benefiting from the differentiability of the Bernstein polynomials, we are able to explicitly define the admissible control space, leading to collision-free movements of the robots with irregular shapes. This yields forward-invariant safety with minimal deviation from nominal commands and scales to dynamic, multi-robot scenes.
Our \textbf{contributions} can be summarized as:
\begin{itemize}
    \item We introduce a novel unified barrier function representation with a single SDF for multiple obstacles using Bernstein polynomials. This approach provides a common template for deriving geometry-aware control constraints for safe navigation in unstructured environments.
    \item We define a novel notion of \emph{Geometry-aware Control Barrier Functions} built on Bernstein polynomial representations and explicitly derive control constraints from the proposed CBFs for arbitrary smooth convex SDF level sets, which yields closed-loop safety certificates.
    \item We include results from the experiments for multi-robot collision avoidance and multi-obstacle robot navigation that verify the effectiveness of the framework.
\end{itemize}

\section{Preliminaries}
Considering a team of $N$ robots moving in the $d$-dimensional space $\mathbb{R}^{d}$, the position of each robot is denoted as $\mathbf{x}_i\in \mathcal{X}_i \subset \mathbb{R}^{d}, \forall i \in \{1,...,N\}$, with dynamics governed by:
\begin{equation}\label{eq:dynamics}
    \dot{\mathbf{x}}_i = f_i(\mathbf{x}_i) + g_i(\mathbf{x}_i)\mathbf{u}_i
\end{equation}
where $f_i:\mathbb{R}^{d}\rightarrow\mathbb{R}^{d}$ and $g_i:\mathbb{R}^{d}\rightarrow\mathbb{R}^{d\times q}$ are locally Lipschitz continuous and $\mathbf{u}_i\in\mathbb{R}^{q}$ is the control input for robot $i$.

\subsection{Control Barrier Functions}
In this section, we briefly review the concept of Control Barrier Functions (CBFs) \cite{ames2019control}.  
Given the control-affine system in Eq.~\eqref{eq:dynamics}, a desired safe set for the state of robot $i$ can be specified as the 0-superlevel set of a continuously differentiable function $h:\mathbb{R}^d \rightarrow \mathbb{R}$,
\begin{equation}\label{eq:h_regular}
 \mathcal{H}_i \;=\; \{\mathbf{x}_i \in \mathbb{R}^{d} \;:\; h(\mathbf{x}_i) \geq 0\}.
\end{equation}

CBFs have been used to design control constraints that enforce the forward invariance of the safe set $\mathcal{H}_i$ for robot $i$, i.e., if the system starts within $\mathcal{H}_i$, it remains there for all future time under satisfying controllers. The key result can be summarized as follows.

\begin{lemma}
\label{lem:cbf}(Summarized from \cite{ames2019control})
Given a dynamical system affine in control and a desired set $\mathcal{H}_i$ as the 0-superlevel set of a \textbf{continuously differentiable} function $h(\mathbf{x}_i): \mathcal{X}_i \rightarrow \mathbb{R}$, the function $h$ is called a control barrier function, if there exists an extended class-$\mathcal{K}$ function\footnote{In the rest of this paper, we select the particular choice of $\kappa(h(\mathbf{x}_i))=\gamma h(\mathbf{x}_i)$ with $\gamma$ as a user-defined parameter \cite{ames2019control}.}  $\kappa(\cdot)$ such that 
$\sup_{\mathbf{u}_i\in\mathcal{U}_i}\{\dot{h}(\mathbf{x}_i, \mathbf{u}_i)\}\geq -\kappa(h(\mathbf{x}_i))$ for all $\mathbf{x}_i\in\mathcal{X}_i$. The admissible control space for any Lipschitz continuous controller $\mathbf{u}_i \in \mathcal{U}_i$ rendering $\mathcal{H}_i$ forward invariant (i.e., keeping the system state $\mathbf{x}_i$ in $\mathcal{H}_i$ over time) thus becomes: 
\begin{align}\label{eq:cbc_lemma}
    \mathcal{B}(\mathbf{x}_i) = \{ \mathbf{u}_i\in \mathcal{U}_i | L_f h(\mathbf{x}_i) + L_gh(\mathbf{x}_i)\mathbf{u}_i + \kappa(h(\mathbf{x}_i))\geq 0 \} \notag
\end{align}
where $L_f$ and $L_g$ denote the Lie derivatives along the vector fields $f$ and $g$, respectively.
\end{lemma}

Existing work \cite{wang2017safety} leverages Lemma~\ref{lem:cbf} to enforce pairwise distance constraints between robots or between robots and obstacles. However, directly encoding such constraints becomes cumbersome in environments with complex or irregular geometries. Note that when multiple CBF constraints are present, the safety guarantee relies on the composition of CBFs. Readers are referred to \cite{wang2017safe,ong2023nonsmooth} for detailed discussion. For systems with higher-order dynamics, the corresponding extension can be found in \cite{xiao2019control}.
This motivates the use of a more general geometric representation of safe sets, which leads us to the notion of Signed Distance Fields (SDF) and Bernstein polynomials discussed next.

\subsection{Signed Distance Fields Representation using Bernstein Polynomials}
Existing work has adopted the \emph{Signed Distance Field} (SDF) \cite{mukadam2018continuous} as a general representation of the safe set of states. This provides an implicit and differentiable description of arbitrary geometries summarized below.

\begin{definition}(Summarized from \cite{mukadam2018continuous})[Signed Distance Field]
For a set $\Omega \subset \mathbb{R}^d$, the signed distance field is defined as
\begin{equation}
    S_{\Omega}(\mathbf{p}) = \pm \inf_{\mathbf{p}' \in \partial \Omega}\|\mathbf{p} - \mathbf{p}'\|_2,
\end{equation}
where $\mathbf{p}$ is a point in the Euclidean space, with $S_{\Omega}(\mathbf{p}) < 0$ if $\mathbf{p} \in \Omega$, $S_{\Omega}(\mathbf{p}) = 0$ if $\mathbf{p} \in \partial \Omega$, and $S_{\Omega}(\mathbf{p}) > 0$ otherwise.
\end{definition}

It is noted that directly storing or evaluating $S_{\Omega}(\mathbf{p})$ for arbitrary shapes can be computationally expensive, and in many cases the complex geometry of $\Omega$ makes an explicit analytic form intractable.  
To overcome this challenge, existing work has employed \emph{Bernstein polynomial} (BP) approximations of SDFs \cite{li2024representing}, which yield a compact and smooth parameterization suitable for optimization and control.

\begin{definition}(Summarized from \cite{calinon2019mixture})[Bernstein Polynomial Basis]\label{def:bp}
Given $\xi \in [0,1]$, let $\phi(\xi)\in\mathbb{R}^Q$ be the vector of basis functions whose $q$-th element, for Bernstein polynomials of order $Q$ (degree $Q\!-\!1$), is
\[
\; \phi_q(\xi) \;=\; \binom{Q-1}{q}\,\xi^{q}(1-\xi)^{\,Q-1-q},\;\; q=0,\ldots,Q-1.
\]
With this, for any point $\mathbf{p} \in \mathbb{R}^d$, the multivariate basis is constructed via the Kronecker product
\begin{equation}\label{eq:basis}
    \Psi(\mathbf{p}) \;=\; \bigotimes_{p=1}^{d} \phi(\xi_p) \;\in\; \mathbb{R}^{Q^{d}},
\end{equation}
where \(\xi_p=\dfrac{[\mathbf{p}]_p-[\mathbf{p}]^{\min}_p}{[\mathbf{p}]^{\max}_p-[\mathbf{p}]^{\min}_p}\in[0,1]\) and \(p=1,\ldots,d\) are the dimension-wise normalized coordinates of $\mathbf{p}$ (i.e., the normalized value of the point's $p$-th dimension), where $[\mathbf{p}]^{\min}$ and $[\mathbf{p}]^{\max}$ cover the maximum relative distance.
\end{definition}

Thus, given Definition~\ref{def:bp}, the SDF is approximated as a weighted combination of the \(Q^d\) tensor-product Bernstein basis functions (order \(Q\) per dimension) \cite{li2024representing}:

\begin{equation}\label{eq:bp-sdf}
    \hat S_{\Omega}(\mathbf{p})
    \;=\; \langle \Psi(\mathbf{p}),\, \mathbf{w} \rangle
    \;=\; \Psi(\mathbf{p})^\top \mathbf{w},
\end{equation}
where $\hat S_{\Omega}(\mathbf{p})\in\mathbb{R}$, \(\Psi(\mathbf{p}) \in \mathbb{R}^{Q^{d}}\) is the multivariate Bernstein basis vector evaluated at the point \(\mathbf{p}\), and
\(\mathbf{w} \in \mathbb{R}^{Q^{d}}\) are the coefficients. $\langle \cdot, \cdot \rangle$ is the operation of the inner product.

Following \cite{li2024representing}, the coefficients $\mathbf{w}$ can be obtained via linear least-squares regression under the assumption that robot shapes are learned offline.  
Suppose we have $K$ sampled points $\{\mathbf{p}^{(j)}, S_{\Omega}(\mathbf{p}^{(j)})\}_{j=1}^K$ with ground-truth signed distances $S_{\Omega}(\mathbf{p}^{(j)})$.  
By stacking the basis evaluations into a design matrix $\mathbf{\Psi} \in \mathbb{R}^{Q^d \times K}$ and the ground-truth distances into $\mathbf{f} \in \mathbb{R}^K$, the linear regression problem is
\begin{equation}
    \mathbf{w}^* = \arg\min_{\mathbf{w}} \|\mathbf{\Psi}^\top\mathbf{w} - \mathbf{f}\|_2^2 
\end{equation}
If $\mathbf{\Psi}$ has full row rank, the closed-form solution can be formally defined as:
\begin{equation}
   \mathbf{w}^* \;=\; \big(\mathbf{\Psi}\,\mathbf{\Psi}^\top\big)^{-1}\,\mathbf{\Psi}\,\mathbf{f}
\end{equation}
With the closed-form coefficients $\mathbf{w}^*$, we can obtain the smooth SDF approximation $\hat S_{\Omega}(\mathbf{p})=\Psi(\mathbf{p})^\top\mathbf{w}^*$ over the workspace. 
However,
obtaining the solution via explicit inversion of a large matrix is unnecessary and often prohibitive, both computationally and in memory, and can be numerically unstable~\cite{hager1989updating}. Readers are referred to~\cite{li2024representing} for more detailed discussions.

\subsection{Problem Statement}
In this paper, our objective is to develop a control framework that \emph{leverages the Bernstein-polynomial approximation} $\hat S_{\Omega}$ of the SDF and \emph{employs control barrier functions (CBFs)} to encode and enforce safety in realistic, complex environments. While the BP-SDF surrogate offers a compact and differentiable geometric representation, translating this approximation into formal safety guarantees remains non-trivial in cluttered scenes and in the presence of interacting agents. With that, given any control bound $\mathbf{u}_\mathrm{max}\in\mathbb{R}$ of each robot $i$, the problem becomes how to minimally modify any pre-defined nominal input $\mathbf{u}^\mathrm{nom}$ with the BP-SDF-defined safe set, i.e., enforcing nonnegative distance between the robot $\hat{S}_{\mathcal{R}}(\cdot)$ and the non-robot $\hat{S}_{\mathcal{O}}(\cdot)$:
\begin{align}\label{eq:problem}
    &\qquad \mathbf{u}^* = \arg\min_\mathbf{u} ||\mathbf{u}-\mathbf{u}^\mathrm{nom}||^2 \\
    \text{s.t.}
    &\quad \mathbf{x}\in \{\text{dist}(\hat S_\mathcal{O}(\cdot)=0, \hat S_\mathcal{R}(\cdot)=0) \geq 0 \} \label{eq:state}\\
     &\qquad \qquad ||\mathbf{u}|| \leq \mathbf{u}_\mathrm{max} \notag
\end{align}

Accordingly, in this paper, we focus on two questions:
\begin{itemize}
    \item \textbf{Single-robot safe navigation.} Given a BP-SDF approximation of the workspace containing multiple irregular-shaped static obstacles,
    how can we transfer the state constraints in Eq.~\eqref{eq:state} to control constraints and design a control framework that can guarantee that the robot safely navigates in the workspace? 
    \item \textbf{Multi-robot extension.} Given a team of $N$ robots where each robot $i$ maintains a robot-specific BP-SDF $\hat S_{\Omega}^{i}$, how can this framework be extended to multiple robots so that collisions among moving agents are avoided?
\end{itemize}

\section{Method}

In this section, we build on the unified SDF representation in Eq.~\eqref{eq:bp-sdf} to formulate novel control barrier functions that directly encode geometry-aware collision avoidance conditions and derive corresponding control constraints. We first present the single-robot SDF barrier certificates and then extend the approach to the multi-robot setting.

\subsection{Unified Barrier Function with Bernstein polynomials}\label{sec:bp_unified}
Let $\mathcal O=\bigcup_{j=1}^{M}o_j$ denote the union of all obstacles. Following the work~\cite{li2024representing}, we use Bernstein polynomials to represent the Signed Distance Field (SDF) of the robot $\mathbf{x}_i$ and obstacles $\mathbf{x}^\mathrm{obs}_o \in \mathcal O $. As shown in Fig.~\ref{fig:kkt_idea}(a) and leveraging Eq.~\eqref{eq:bp-sdf}, we represent the robot $i$ by an SDF $\hat S^{i}_{\mathcal R}(\cdot)$ and the composite obstacle set by a single \emph{unified} SDF $\hat S_{\mathcal O}(\cdot)$.

\begin{equation}
    \hat S^{i}_{\mathcal R}(\mathbf{p}_i) = \langle\Psi(\mathbf{p}_i),\mathbf{w}_i\rangle, \quad
    \hat S_{\mathcal O}(\mathbf{p}^\mathrm{obs}_o) = \langle\Psi(\mathbf{p}^\mathrm{obs}_o),\mathbf{w}_o\rangle \notag
\end{equation}
where $\hat S^{i}_{\mathcal R}(\cdot)$ is in the robot's local frame and $\hat S_{\mathcal O}(\cdot)$ is in the unified obstacle frame. $\mathbf{p}_i$ and $\mathbf{p}_{o}^{\mathrm{obs}}$ represent the point $\mathbf{p}$ in the robot frame and unified obstacle frame.

\begin{figure}[t]
    \centering  \includegraphics[width=\linewidth]{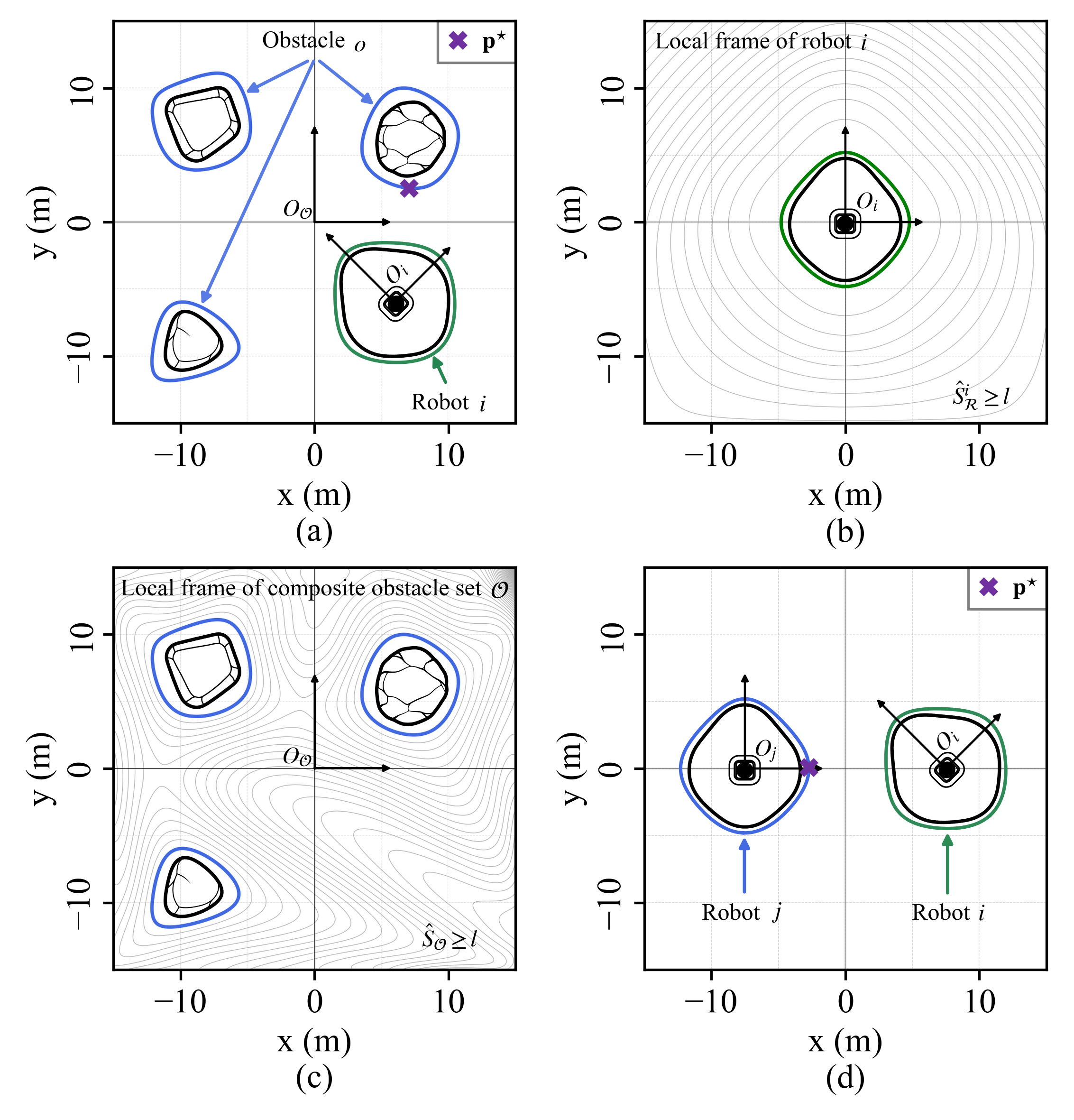}
    \caption{Examples of the level sets for SDF approximation using Bernstein polynomials on robots and obstacles.
    \textbf{(a)} 
    With a safety margin $l\in\mathbb{R}$, the green contour represents the $l$-level set of the approximated SDF of the robot $i$. The three blue contours represent the $l$-level set of the approximated SDF w.r.t. all three obstacles in the workspace using a single Bernstein polynomial representation.
    \textbf{(b)} The green contour
    represents the $l$-level set of the approximated SDF of the robot $i$ in the robot frame $O_i$, satisfying $\hat{S}^{i}_{\mathcal{R}} \geq l$.
    \textbf{(c)} The blue contours represent the $l$-level set of the approximated SDF of the composite obstacle set $\mathcal{O}$ in the frame $O_\mathcal{O}$, satisfying $\hat{S}_{\mathcal{O}} \geq l$.
    \textbf{(d)} The green and blue contours represent the $l$-level sets of the approximated SDF for robot $i$ and $j$.
    }   
    \label{fig:kkt_idea}
\end{figure}

We define the function $\Phi(\mathbf{x}, \mathbf{p}_{\mathrm{local}})$ to transform the point in the local frame $p_{\text{local}}$ to the global frame $p$. With this,

\begin{equation}
    F_{i}(\mathbf{x}_i, \mathbf{p}) = \hat S_{\mathcal R}^{i}(\Phi_i^{-1}(\mathbf{x}_i, \mathbf{p}))
\end{equation}
represents the distance between any point $\mathbf{p}$ and the boundary of the robot $i$ in the space. 
 
Therefore, $F_{i}(\mathbf{x}_i, \mathbf{p}) = 0$ can be used to define the zero level set for the robot $i$ given its SDF. 

Inspired by~\cite{dai2024sailing}, we first characterize the safe space using the minimum distance between the robot and the composite obstacle set. To this end, we formulate the following optimization problem to find the closest point $\mathbf{p}^*$ on the obstacle boundary to the robot:
\begin{align}\label{eq: h_sol}
    \mathbf{p}^*_{i,o} &= \arg\min_{\mathbf{p}} \;F_{i}(\mathbf{x}_i, \mathbf{p})  \\
     & \textnormal{s.t.} \quad F_{o}(\mathbf{x}_o^{\mathrm{obs}}, \mathbf{p}) = 0 \notag
\end{align}
As $F_i(\mathbf{x}_i,\mathbf{p})$ denotes the signed distance from $\mathbf{p}$ to the boundary of robot $i$, 
minimizing it corresponds to finding the closest point to the robot. And the constraint $F_o(\mathbf{x}_o^{\mathrm{obs}},\mathbf{p})=0$ enforces that $\mathbf{p}$ lies on the obstacle boundary. Therefore, as illustrated in Fig.~\ref{fig:kkt_idea}, the solution $\mathbf{p}^*_{i,o}$ is exactly the obstacle boundary point with the minimum distance from the robot, which is then used to define the barrier function for collision avoidance. By solving the optimization problem in~\eqref{eq: h_sol}, one can obtain the optimal solution $\mathbf{p}_{i,o}^*$, which is the closest point on the obstacle boundary to robot $i$.

Hence, to formally characterize safety, we directly utilize the property of the BP-SDF, which returns the signed distance from the robot to any given point. By evaluating the robot's SDF at the closest obstacle-boundary point $\mathbf{p}_{i,o}^*$, we formally define the \textbf{Geometry-aware Control Barrier Function (CBF)} as:
\begin{equation}\label{eq:h_signed}
    h(\mathbf{x}_i,\mathbf{x}_o^{\mathrm{obs}})
    =
    F_i(\mathbf{x}_i,\mathbf{p}_{i,o}^*)
\end{equation}
Therefore, $h(\mathbf{x}_i,\mathbf{x}_o^{\mathrm{obs}})\ge 0$ indicates that robot $i$ remains away from the obstacle, while $h(\mathbf{x}_i,\mathbf{x}_o^{\mathrm{obs}})< 0$ indicates overlap between the robot and the obstacle. Accordingly, the safe set is defined as
\begin{equation}\label{eq:our_safe_set}
    \mathcal{H}_{i,o}^{\mathrm{obs}}
    =
    \left\{
    (\mathbf{x}_i,\mathbf{x}_o^{\mathrm{obs}})\in\mathbb{R}^{2d}
    \mid
    h(\mathbf{x}_i,\mathbf{x}_o^{\mathrm{obs}})\ge 0
    \right\}.
\end{equation}

\subsection{Single-robot Geometry-aware Control Barrier Certificates}\label{sec:single_sdf_cbf}
By using the KKT conditions to solve the problem in Eq.~\eqref{eq: h_sol}, 
we can obtain the optimal value of the barrier function $h$ and the closest point to the robot, $\mathbf{p}^*$. We propose the following proposition to define the safe control space.

\begin{proposition} [\textbf{Geometry-aware Control Barrier Certificates}]\label{proposition:geo_cbf}
    Consider robot $i$ with control-affine dynamics ~\eqref{eq:dynamics} moving in a workspace with arbitrary-shaped static obstacles
    which are enclosed by a unified SDF $\hat{S}_{\mathcal{O}}$. The barrier function is $h(\mathbf{x}_i, \mathbf{x}_o^{\mathrm{obs}})$ and the corresponding obstacle boundary point is $\mathbf{p}_{i,o}^*$ in Eq.~\eqref{eq: h_sol}. Following Lemma~\ref {lem:cbf}, for the robot $i$ and the environment with obstacles $\mathcal{O}$, the admissible control space for any Lipschitz continuous controller rendering the safe set in Eq.~\eqref{eq:our_safe_set} forward invariant thus becomes:
    \begin{equation}\label{eq: prop}
    \begin{aligned}
       \mathcal{B}^\mathrm{obs}_\mathbf{u}= \big\{ \mathbf{u}_i \in \mathcal{U}_i \mid \; & a_{i}(\mathbf{x}_i, \mathbf{p}_{i,o}^*) + A_{i}(\mathbf{x}_i, \mathbf{p}_{i,o}^*) \mathbf{u}_i \\
       & \geq -\gamma h(\mathbf{x}_i, \mathbf{x}_o^{\mathrm{obs}}) \big\}
    \end{aligned}
    \end{equation}
    where $a_{i}(\mathbf{x}_i, \mathbf{p}_{i,o}^*) = \nabla_{\mathbf{x}_i}F_{i}(\mathbf{x}_i, \mathbf{p}_{i,o}^*)^{\top}f_i(\mathbf{x}_i)$ and $A_{i}(\mathbf{x}_i, \mathbf{p}_{i,o}^*) = \nabla_{\mathbf{x}_i}F_{i}(\mathbf{x}_i, \mathbf{p}_{i,o}^*)^{\top}g_i(\mathbf{x}_i)$.
\end{proposition}

\begin{proof}
   We assume the BP-SDF level sets are convex such that the closest point $\mathbf{p}_{i,o}^*$ changes smoothly without discontinuous jumps. Additionally, $||\nabla F_i|| \neq 0$ on the boundary of the BP-SDF learned offline.

   With this, differentiating Eq.~\eqref{eq:h_signed} with respect to time $t$ using the chain rule yields:
    \begin{equation*} 
    \dot{h}(\mathbf{x}_i, \mathbf{x}_o^{\mathrm{obs}}) = \nabla_{\mathbf{x}_i} F_i(\mathbf{x}_i, \mathbf{p}_{i,o}^*)^{\top} \dot{\mathbf{x}}_i + \nabla_{\mathbf{p}} F_i(\mathbf{x}_i, \mathbf{p}_{i,o}^*)^{\top} \dot{\mathbf{p}}_{i,o}^*
    \end{equation*}
   Based on the KKT conditions of the optimization problem in Eq.~\eqref{eq: h_sol}, the spatial gradient $\nabla_{\mathbf{p}} F_i(\mathbf{x}_i, \mathbf{p}^*)$ is parallel to the normal vector of the obstacle boundary at $\mathbf{p}^*$. Since $\dot{\mathbf{p}}^*$ is tangent to the obstacle boundary, their inner product is zero, i.e., $\nabla_{\mathbf{p}} F_i(\mathbf{x}_i, \mathbf{p}^*)^{\top} \dot{\mathbf{p}}^* = 0$.

   Substituting this identity gives us the general form:
   \begin{equation}
    \dot{h}(\mathbf{x}_i, \mathbf{x}_o^{\mathrm{obs}}) = \nabla_{\mathbf{x}_i} F_i(\mathbf{x}_i, \mathbf{p}_{i,o}^*)^{\top} \dot{\mathbf{x}}_i
\end{equation}
   
   Following Lemma~\ref{lem:cbf}, the safe control space can thus be given by Eq.~\eqref{eq: prop}.
\end{proof}
This defines the admissible control space that ensures that the robot stays away from the obstacle boundary.
Finally, one can define the safe controller $\mathbf{u}_i^{*}$ for a robot navigating in a multi-obstacle scenario by minimally modifying the nominal control $\mathbf{u}^\mathrm{nom}_i$ with Quadratic Programming (QP):
\begin{align} \label{eq:QP single robot}
    \mathbf{u}^{*}_i=&\arg\min_{\mathbf{u}_i} \quad ||\mathbf{u}_i - \mathbf{u}^{\text{nom}}_i||_2^2 \\
    \text{s.t.} \quad a_{i}(\mathbf{x}_i, \mathbf{p}_{i,o}^*) &+ A_{i}(\mathbf{x}_i, \mathbf{p}_{i,o}^*) \mathbf{u}_i  \geq -\gamma h(\mathbf{x}_i, \mathbf{x}_o^{\mathrm{obs}}) \notag\\
     &||\mathbf{u}_i|| \leq \mathbf{u}_i^\mathrm{max} \notag
\end{align}
where $\mathbf{u}_i^\mathrm{max}$ denotes the control bound for robot $i$.

\subsection{Handling Non-Smoothness in Composite Obstacles}

As defined in Section~\ref{sec:bp_unified}, representing multiple obstacles as a single composite set $\mathcal{O}$ provides a unified geometric constraint. However, in such a case, the closest point $\mathbf{p}^*_{i,o}$ on the obstacle boundary can discontinuously jump between different obstacles while the robot is moving. This can make the distance function non-smooth, which violates the continuous differentiability property of CBFs. 

\begin{figure}[t]
    \centering
\includegraphics[width=\linewidth]{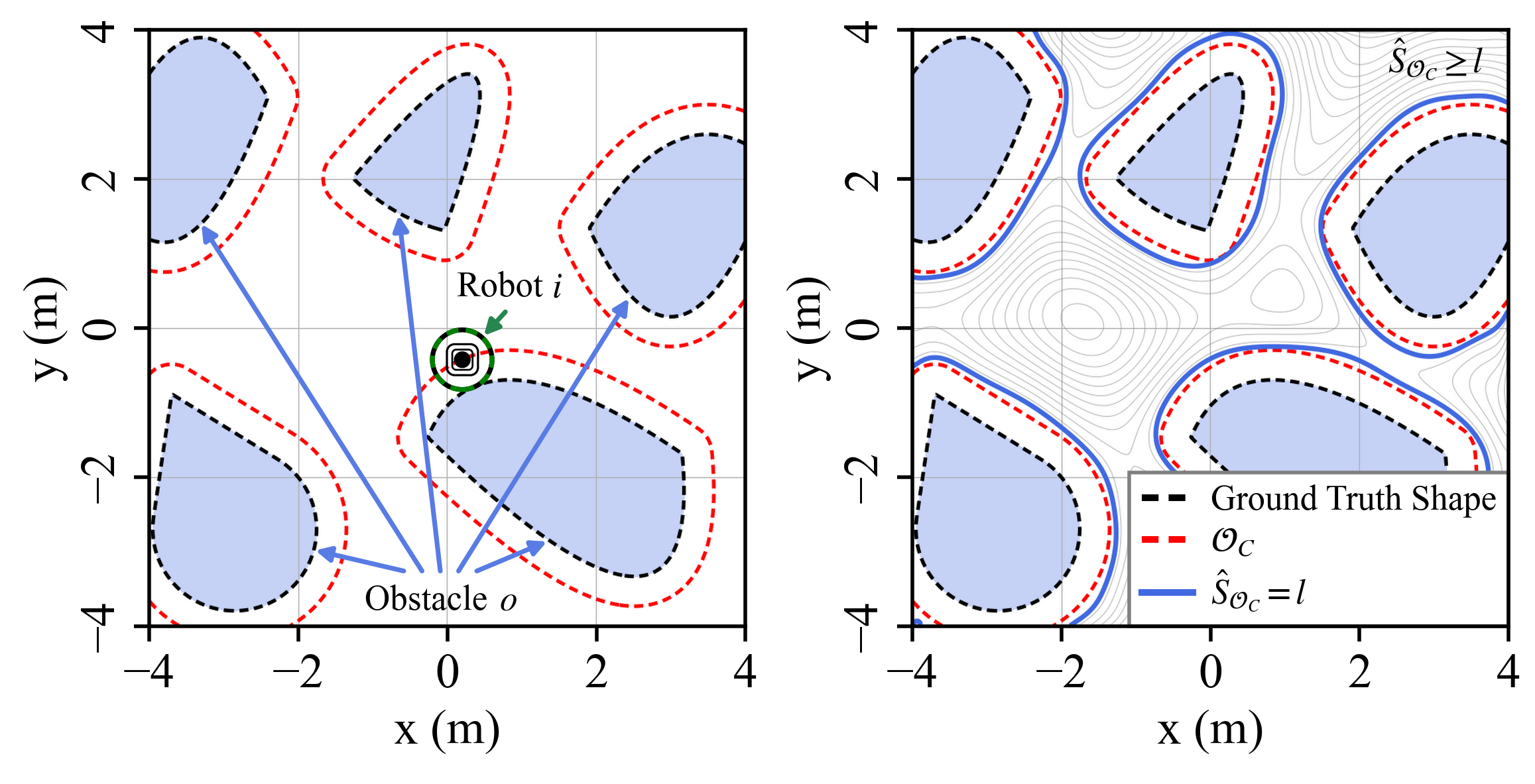}
    \caption{C-space (red) of the composite obstacle set $\mathcal{O}_C$ inflated by the radius $r$ of the robot $i$ (green). The right figure shows the SDF $\hat{S}_{\mathcal{O}_C}$ of the C-space (blue contours) with a user-defined safety margin level set $l$.}
    \label{fig:cspace}
\end{figure}

To deal with the composite set of obstacles scenario while maintaining smoothness, we adopt a Configuration Space (C-space) approach. By approximating the robot as a bounding circle of radius $r$ \cite{wang2017safety}, we inflate the ground-truth composite obstacles using the Minkowski sum to create the C-space obstacles, denoted as $\mathcal{O}_{C} = \mathcal{O} \oplus \mathcal{D}(r)$, where $\mathcal{D}(r)$ represents a disk of radius $r$.
We then construct a unified BP-SDF for the C-space, denoted as $\hat{S}_{\mathcal{O}_{C}}(\cdot)$. However, if this causes overly conservative behaviors of the robot, we can define each obstacle as an individual safety constraint.

In this C-space formulation, the robot is treated as a point $\mathbf{x}_i$. The barrier function is defined by the BP-SDF of the C-space at the robot's position:
\begin{equation}
    h_{i}^{\mathrm{obs}}(\mathbf{x}_i) = \hat{S}_{\mathcal{O}_{C}}(\mathbf{x}_i) = \Psi(\mathbf{x}_i)^\top \mathbf{w}_o
\end{equation}
Because the Bernstein polynomial approximation is inherently smooth and continuously differentiable, evaluating $\hat{S}_{\mathcal{O}_{C}}(\mathbf{x}_i)$  guarantees a smooth CBF everywhere in the workspace.

\subsection{Multi-robot Geometry-aware Control Barrier Certificates}
We continue to extend our Geometry-aware Control Barrier Function in~\eqref{eq:h_signed} for the multi-robot collision avoidance case. Consider a team of $N$ robots. Similarly to~\eqref{eq: h_sol}, for any pair of robots $i$ and $j$, one can calculate the 
closest point $\mathbf{p}_{i,j}^*$ on the boundary of $j$ to robot $i$. 
Then for any pair-wise collision avoidance in the robot team, the safety of $\mathbf{x} = \{\mathbf{x}_1, \cdots, \mathbf{x}_i, \cdots, \mathbf{x}_N\}$ is defined as:
\begin{equation}\label{eq:cbf_multi}
h(\mathbf{x}_i, \mathbf{x}_j) = F_i(\mathbf{x}_i, \mathbf{p}_{i,j}^*), \forall i \neq j
\end{equation}

With this, $h(\mathbf{x}_i, \mathbf{x}_j) \ge 0$ indicates that, from robot $i$'s view, the robot $j$ is sufficiently far from robot $i$. Accordingly, the safe set is defined as
\begin{equation}\label{eq:our_safe_set_multi}
    \mathcal{H}_{i,j}
    =
    \left\{ \mathbf{x}\in\mathbb{R}^{dN}
    \mid
    h(\mathbf{x}_i,\mathbf{x}_j)\ge 0
    \right\}, \forall i \neq j.
\end{equation}

\begin{proposition} [\textbf{Multi-robot Geometry-aware Control Barrier Certificates}]\label{proposition:Multi_cbf}
Considering a team of $N$ robots moving in the same workspace, each robot has its own robot-specific SDF approximated by Bernstein polynomials. $\mathbf{u} = \{\mathbf{u}_1, \cdots,\mathbf{u}_{N}\}$ is the joint control of the $N$ robots.
As shown in Fig.~\ref{fig:kkt_idea}(d), for any pair of robots $i$ and $j$, the optimal point $\mathbf{p}_{i,j}^*$ moves along the surface of robot $j$. By taking the time derivative of the boundary constraint $F_j(\mathbf{x}_j, \mathbf{p}_{i,j}^*) = 0$ and applying the KKT conditions, the inter-robot collision avoidance can be enforced using the following \textbf{Multi-robot Geometry-aware Control Barrier Certificates}:
\begin{equation}\label{eq: robots_prop}
\begin{split}
    \mathcal{B}^\mathrm{robots}_\mathbf{u} = \big\{ \mathbf{u} \in \mathcal{U} \mid \ 
    & a_{i}(\mathbf{x}_i, \mathbf{p}_{i,j}^*) + a_{j}(\mathbf{x}_j, \mathbf{p}_{i,j}^*) \\
    & + A_{i}(\mathbf{x}_i, \mathbf{p}_{i,j}^*)\mathbf{u}_i + A_{j}(\mathbf{x}_j, \mathbf{p}_{i,j}^*) \mathbf{u}_j \\
    & \geq -\gamma h(\mathbf{x}_i, \mathbf{x}_j) \big\}
\end{split}
\end{equation}

where $a_{j}(\mathbf{x}_j, \mathbf{p}_{i,j}^*) = \lambda \nabla_{\mathbf{x}_j} F_{j}(\mathbf{x}_j, \mathbf{p}_{i,j}^*)^{\top} f_j(\mathbf{x}_j)$,
$A_{j}(\mathbf{x}_j, \mathbf{p}_{i,j}^*) = \lambda \nabla_{\mathbf{x}_j} F_{j}(\mathbf{x}_j, \mathbf{p}_{i,j}^*)^{\top} g_j(\mathbf{x}_j)$,
and the scalar multiplier $\lambda = \frac{\lVert \nabla_{\mathbf{p}} F_{i}(\mathbf{x}_i, \mathbf{p}_{i,j}^*) \rVert}{\lVert \nabla_{\mathbf{p}} F_{j}(\mathbf{x}_j, \mathbf{p}_{i,j}^*) \rVert}$.

\end{proposition}
\begin{proof}
Taking the time derivative of $h(\mathbf{x}_i, \mathbf{x}_j) = F_i(\mathbf{x}_i, \mathbf{p}_{i,j}^*)$ yields $\dot{h}(\mathbf{x}_i, \mathbf{x}_j) = \nabla_{\mathbf{x}_i} F_i^{\top} \dot{\mathbf{x}}_i + \nabla_{\mathbf{p}} F_i^{\top} \dot{\mathbf{p}}_{i,j}^*$. From the KKT stationarity conditions, we have $\nabla_{\mathbf{p}} F_i = - \lambda \nabla_{\mathbf{p}} F_j$. Substituting this gives $\dot{h}(\mathbf{x}_i, \mathbf{x}_j) = \nabla_{\mathbf{x}_i} F_i^{\top} \dot{\mathbf{x}}_i - \lambda \nabla_{\mathbf{p}} F_j^{\top} \dot{\mathbf{p}}_{i,j}^*$.

Because $\mathbf{p}_{i,j}^*$ must remain on the boundary of robot $j$, the constraint $F_j(\mathbf{x}_j, \mathbf{p}_{i,j}^*) = 0$ is satisfied at all times. Differentiating this gives $\nabla_{\mathbf{x}_j} F_j^{\top} \dot{\mathbf{x}}_j + \nabla_{\mathbf{p}} F_j^{\top} \dot{\mathbf{p}}_{i,j}^* = 0$, which allows us to substitute out the velocity $\dot{\mathbf{p}}_{i,j}^*$, yielding:
\begin{equation}
    \dot{h}(\mathbf{x}_i, \mathbf{x}_j) = \nabla_{\mathbf{x}_i} F_i(\mathbf{x}_i, \mathbf{p}_{i,j}^*)^{\top} \dot{\mathbf{x}}_i + \lambda \nabla_{\mathbf{x}_j} F_j(\mathbf{x}_j, \mathbf{p}_{i,j}^*)^{\top} \dot{\mathbf{x}}_j
\end{equation}
Following Lemma~\ref{lem:cbf}, one can obtain the safe control space for the multi-robot system in \eqref{eq: robots_prop}.
\end{proof}

With this, one can define the following safe controller $\mathbf{u}^{*}$ for the team of robots by minimally modifying their nominal control $\mathbf{u}^\mathrm{nom}$ with the following QP: 
\begin{align}\label{eq: QP multi-robot}
    & \mathbf{u}^{*} =  \arg\min_\mathbf{u} \quad ||\mathbf{u} - \mathbf{u}^{\text{nom}}||^2 \\
    \text{s.t.} \quad &a_{i}(\mathbf{x}_i, \mathbf{p}^*_{i,j}) + a_{j}(\mathbf{x}_j, \mathbf{p}^*_{i,j}) + A_{i}(\mathbf{x}_i, \mathbf{p}^*_{i,j})\mathbf{u}_i \notag \\
    &+ A_{j}(\mathbf{x}_j, \mathbf{p}^*_{i,j}) \mathbf{u}_j \geq -\gamma h(\mathbf{x}_i, \mathbf{x}_j), \forall i \neq j \notag\\
    &\quad ||\mathbf{u}_i|| \leq \mathbf{u}_i^\mathrm{max},\quad \forall i = 1,...,N \notag
\end{align}

\begin{remark}\label{remark:convex}
    While it has been proven that a minimum distance function between two strongly convex sets is a smooth function \cite{thirugnanam2025control}, the learned SDFs are approximations, and some SDF level sets can be non-convex even if the ground truth shape is convex. To handle the non-smoothness, one can formulate a constrained Quadratic Program (QP) for offline learning and strictly enforce convexity on all level sets by constraining the second-order difference matrix of the Bernstein weights $w$ to be positive semi-definite. 
\end{remark}

\begin{remark} \label{rem: l level set}
    It is possible that the learned SDFs with Bernstein polynomials and the method in~\cite{li2024representing} have a small approximation error. 
    A user-defined safety margin level set $l$ can be used to guarantee the enclosure of the ground truth shape in the learned shape by changing the constraints in Eq.~\eqref{eq: h_sol} as $F_{o}(\mathbf{x}_o^{\mathrm{obs}}, \mathbf{p}) -l = 0$ ($l \geq 0 \in \mathbb{R}$).
\end{remark}

\begin{remark}
    It is possible that our method would get into a deadlock. Readers can refer to~\cite{reis2020control, zhang2025adaptive, Lee2025MerryGoRound} for single-robot and multi-robot deadlock resolution.
\end{remark}

\section{Results}

In this section, we conduct simulations to demonstrate the robustness and versatility of the proposed framework. By evaluating its performance across diverse environments, robot dynamics, and homogeneous as well as heterogeneous teams, we establish its broad applicability in ensuring safe and reliable multi-robot coordination. Note that the proposed algorithm is evaluated under \emph{different} system dynamics, including single-integrator and unicycle dynamics.

\begin{figure}[H]
    \centering
    \includegraphics[width=\linewidth]{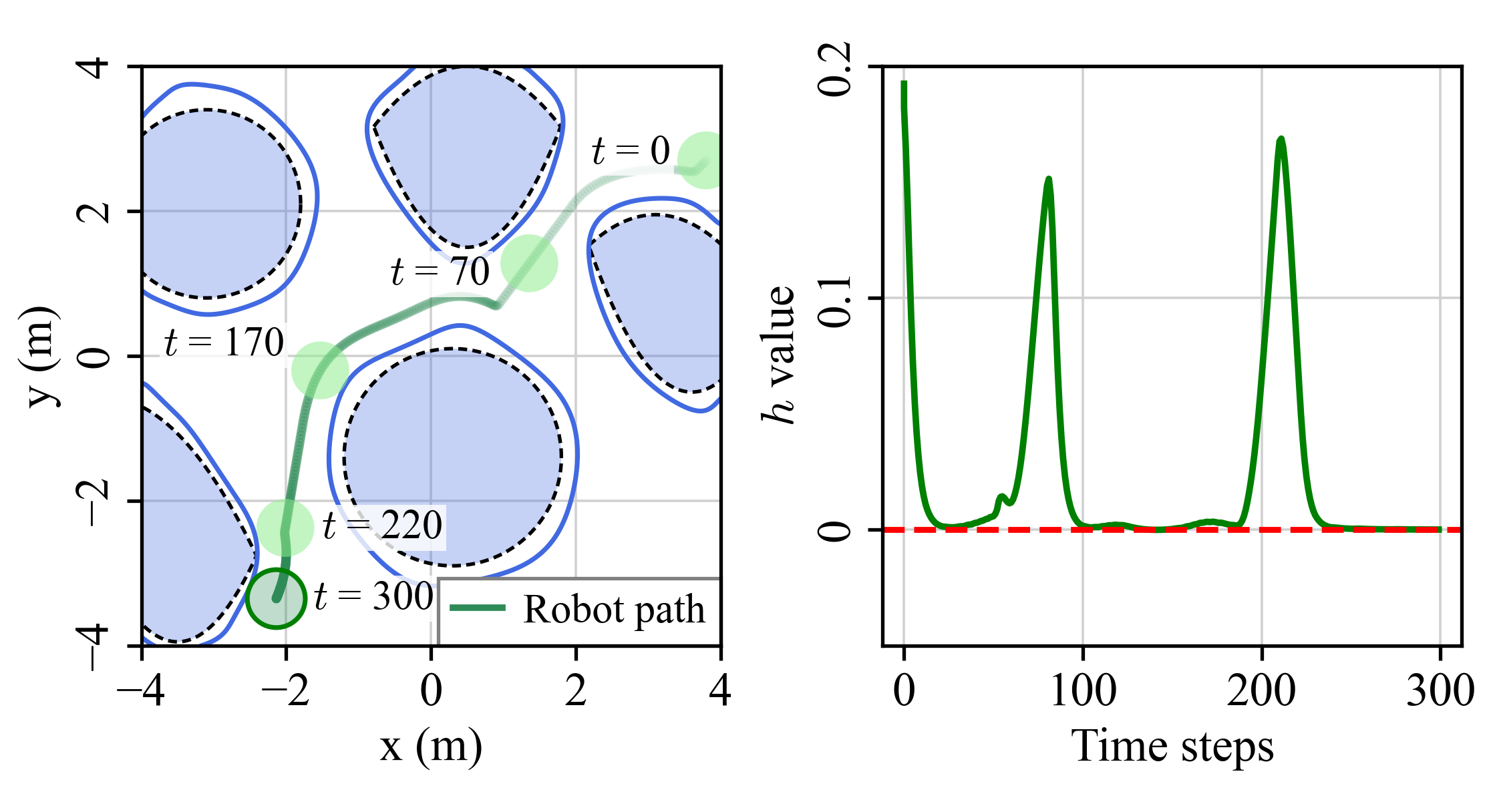}
    \caption{\textbf{Left}: A single BP-SDF (Q = 23) is used to represent five distinct static obstacles, and its 0.27-level set is used to enclose the ground-truth shapes of the static obstacles. The green line shows the robot trajectory starting from $(3.8, 2.7)$ and visiting the designated goal position $(-2.2, -3.5)$.
     \textbf{Right}: The green line represents the value of the barrier function over time steps. The barrier function is obtained from a single CBF constructed using a BP-SDF, which compactly encodes all obstacle-related safety constraints into one unified condition.}
    \label{fig:multi_obstacle}
\end{figure}

In the first experiment, we consider a single robot navigating through an environment populated with multiple obstacles of varying shapes. The robot dynamics are modeled as a single-integrator system,
such that $\dot{\mathbf{x}}_i = \mathbf{u}_i$, 
where $\mathbf{x}_i \in \mathbb{R}^2$ denotes the robot position and $\mathbf{u}_i \in \mathbb{R}^2$ is the control input. The navigation objective is to reach the designated time-invariant goal position while avoiding collisions with obstacles. The nominal task controller is $\mathbf{u}^\mathrm{nom} = -k(\mathbf{x}_i - \mathbf{x}^{\mathrm{goal}}_i)$, where $k$ is the control gain. As illustrated in Figure~\ref{fig:multi_obstacle}, the robot starts at $\mathbf{x} = [3.8, 2.7]^\top$ and visits the goal $\mathbf{x}^{\mathrm{goal}} = [-2.2,-3.5]^\top$.

To enforce safety, we construct the control barrier function (CBF) as in Eq.~\eqref{eq:h_signed}. $\mathbf{p}^\ast$ lies on the $0.27$-level set (setting $l=0.27$ as stated in Remark~\ref{rem: l level set}) of the signed distance field approximated by a Bernstein polynomial with order $Q=23$. The use of the $0.27$-level set ensures that all five obstacles are fully enclosed within a \emph{single} unified BP-SDF. In Figure~\ref{fig:multi_obstacle}, the dashed contours indicate the ground-truth obstacle boundaries, while the fading green curve in the left subfigure illustrates the robot's trajectory. The right subfigure further depicts the temporal evolution of the barrier function, confirming that the safety constraint is satisfied throughout the navigation process.

\subsection{Safe Navigation of Two Robots under Unicycle Dynamics} \label{sec: two robot case}

In the second set of experiments, we consider robots governed by unicycle dynamics:
\begin{equation}
\dot{\mathbf{x}}_i=\begin{bmatrix}
\dot{x}_i \\
\dot{y}_i \\
\dot{\theta}_i
\end{bmatrix}
=
\begin{bmatrix}
\cos(\theta_i) & 0 \\
\sin(\theta_i) & 0 \\
0 & 1
\end{bmatrix}
\begin{bmatrix}
v_i \\
\omega_i
\end{bmatrix},
\end{equation}
where $[x_i,y_i,\theta_i]^{T}$ denotes the state of robot $i$ in the global frame, $v_i$ is the linear velocity, and $\omega_i$ is the angular velocity. The control limits are $|v_i| \leq 1$ m/s and $|\omega_i| \leq \pi/2$ rad/s.  

Given the robot's current state $(x_i,y_i,\theta_i)$ and its corresponding time-invariant goal position $(x_i^g, y_i^g)$, the nominal task controller is designed as
$\hat{v}_i
 = k_v \rho_i, \; \hat{\omega}_i = k_\omega \alpha_i,$
where $k_v,k_\omega > 0$ are proportional gains. 
Specifically, $\rho_i$ denotes the Euclidean distance to the goal and $\alpha_i$ the heading error, given by
$\rho_i = \sqrt{(x_i^g - x_i)^2 + (y_i^g - y_i)^2}, \;
\theta_i^g = \operatorname{atan2}(y_i^g - y_i,\, x_i^g - x_i), \;
\alpha_i = \theta_i^g - \theta_i$. Alternatively, one may employ well-developed unicycle controllers, e.g., \cite{samson2002control}, to construct the nominal controller $\mathbf{u}^\mathrm{nom}$ and drive the robot toward the desired goal. Note that a detailed discussion of such nominal controller design is outside the scope of this paper.

Fig.~\ref{fig:experiment2_sim} illustrates the two-robot position-switching task: the figures show the robot motions with approximated shapes obtained from BP-SDFs. Fig.~\ref{fig:experiment_2_analysis} shows the linear and angular velocity inputs and their respective limits, together with the evolution of barrier functions. The results confirm that the robots complete the task without collision under different dynamics, while satisfying both the safety constraint and control bounds.

\begin{figure*}
    \centering
\includegraphics[width=0.9\linewidth]{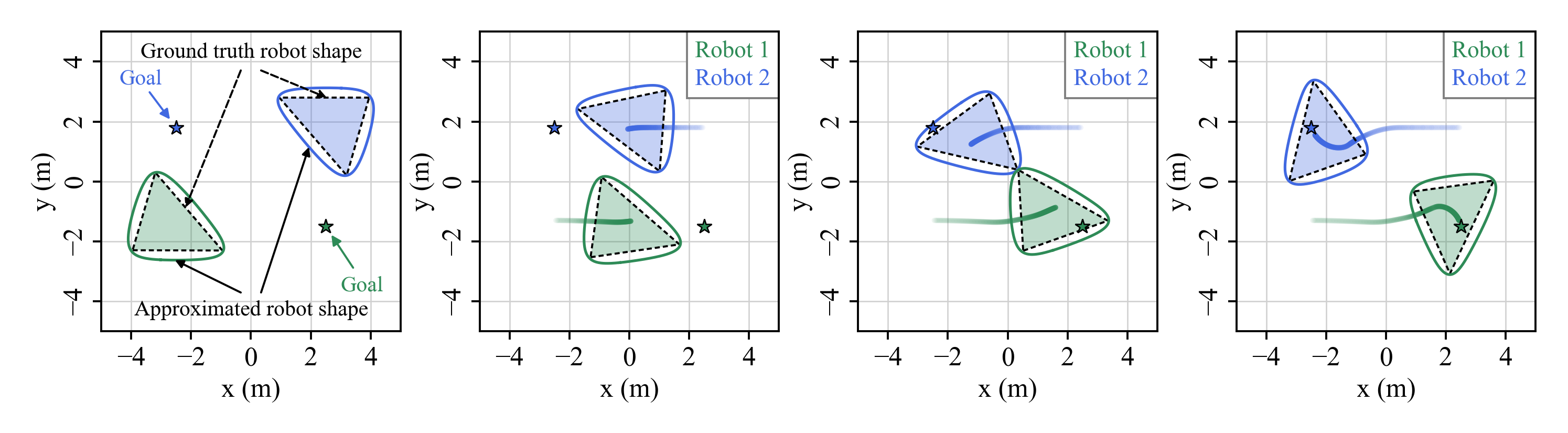}
    \caption{Two-robot position-switching under unicycle dynamics. 
The solid lines and dashed lines represent the approximated shapes by BP-SDFs and ground truth shapes of the robots, respectively. Snapshots at $t=0$, $50$, $100$, and $150$ are shown as representative time instances. 
The visualizations indicate that the robots maintain collision-free execution throughout the entire position-switching process.
}
    \label{fig:experiment2_sim}
\end{figure*}

\begin{figure*}
    \centering
    \includegraphics[width=0.95\linewidth]{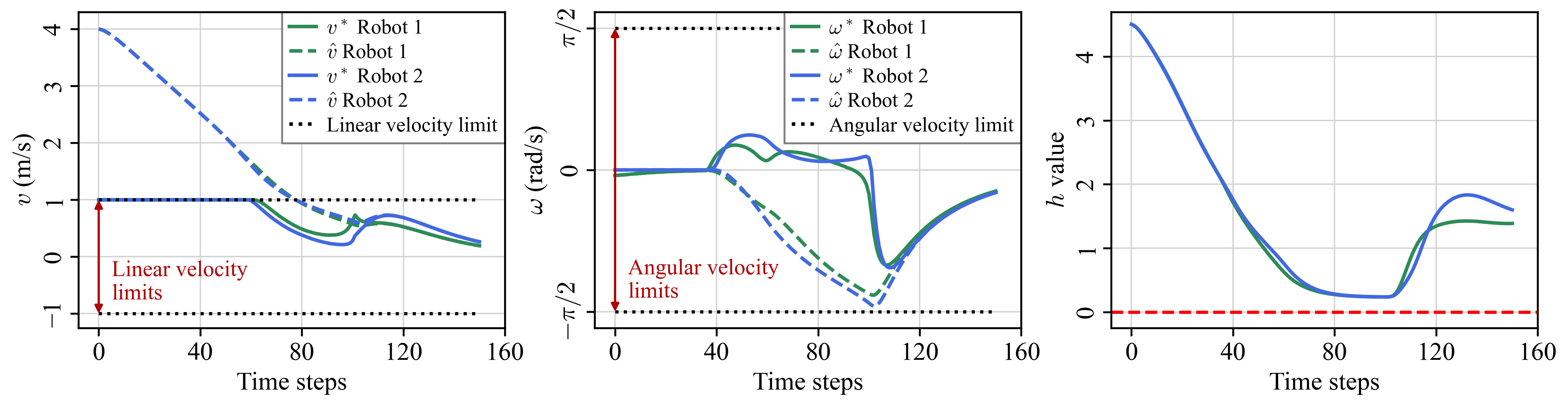}
    \caption{Control input and safety analysis for the two-robot position-switching task. 
\textbf{Left}: linear velocity inputs $v^{*}$ (solid) compared with nominal task inputs $\hat{v}$ (dashed) for both robots, together with the linear velocity limits. 
\textbf{Middle}: angular velocity inputs $\omega^*$ (solid) compared with nominal task inputs $\hat{\omega}$ (dashed), with the corresponding angular velocity limits. 
\textbf{Right}: evolution of the barrier function over time, which remains nonnegative, verifying that the safety constraint is satisfied throughout the entire execution.}
    \label{fig:experiment_2_analysis}
\end{figure*}

\subsection{Safe Navigation of Four Heterogeneous Robots}
\begin{figure*}
    \centering
    \includegraphics[width=0.9\linewidth]{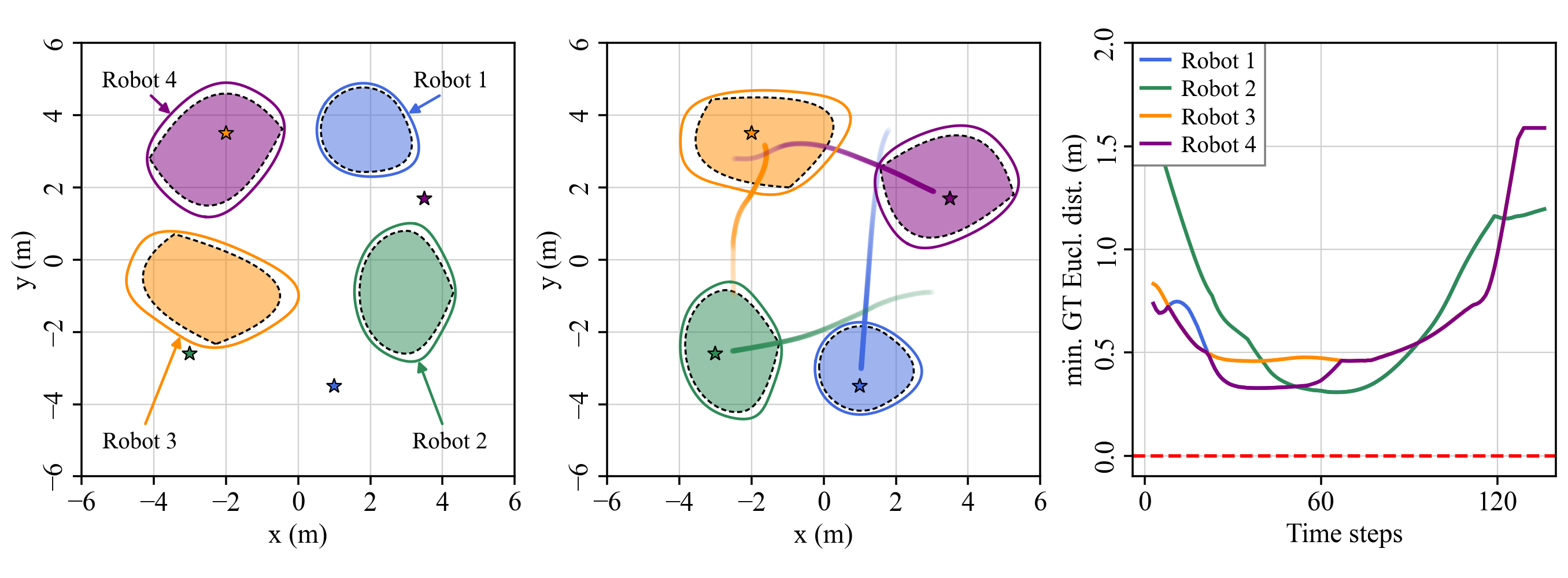}
        \caption{Collision avoidance among four mobile robots with different smooth convex shapes represented by the respective SDF level sets. The right figure shows the minimum Euclidean distance between the ground truth shapes remaining positive over time, which validates the safety guarantees as the robots move to the designated goal positions (marked by stars with the same colors as the robots). 
        }
    \label{fig:exe_3} 
\end{figure*}

The robot dynamics and the design of the task-related controller in this experiment follow the same formulation as in Section~\ref{sec: two robot case}.  
In Fig.~\ref{fig:exe_3}, the heterogeneous example involves four robots with distinct geometries, each approximated by its own SDF. Accordingly, four separate control barrier functions ($h_1, h_2, h_3, h_4$) are required. The results in Fig.~\ref{fig:experiment2_sim} and Fig.~\ref{fig:exe_3}  demonstrate that the proposed framework extends naturally from homogeneous to heterogeneous multi-robot teams while ensuring safety.

It is worth noting that, as illustrated in Fig.~\ref{fig:exe_3}, heterogeneous robot shapes are difficult to approximate using Minkowski-sum-based approaches, which typically rely on conservative simplifications such as bounding circles or convex polygons \cite{11312188}. In contrast, our framework introduces a new CBF formulation that ensures safety under BP-SDF-based modeling of complex and realistic robot geometries.

\subsection{Quantitative Results}
We evaluate our framework on a desktop with 64 GB of RAM and an AMD Ryzen 9 3900X processor. For each obstacle count \(m \in \{1,2,4,6,\ldots,24\}\), we conduct five independent randomized trials, each with a \emph{distinct} obstacle configuration and a \emph{different} start and goal position pair, and report two summary metrics: (i) the average per-timestep computation time \(\bar{\tau}\) (s/step) and (ii) the ground-truth minimum robot–obstacle distance \(d^\star\) (true set-to-set Euclidean distance). Fig.~\ref{fig:quantative_result} shows mean curves with shaded area: in the \textbf{top} panel, the band denotes \(\pm 1\) standard deviation and \(\bar{\tau}\) shows slow growth as \(m\) increases, indicating the efficiency of our proposed method; in the \textbf{bottom} panel, the band also denotes \(\pm 1\) standard deviation and \(d^\star>0\) for all \(m\), demonstrating collision-free operation and robustness in complex environments.

\begin{figure}[htbp]
    \centering
\includegraphics[width=0.7\linewidth]{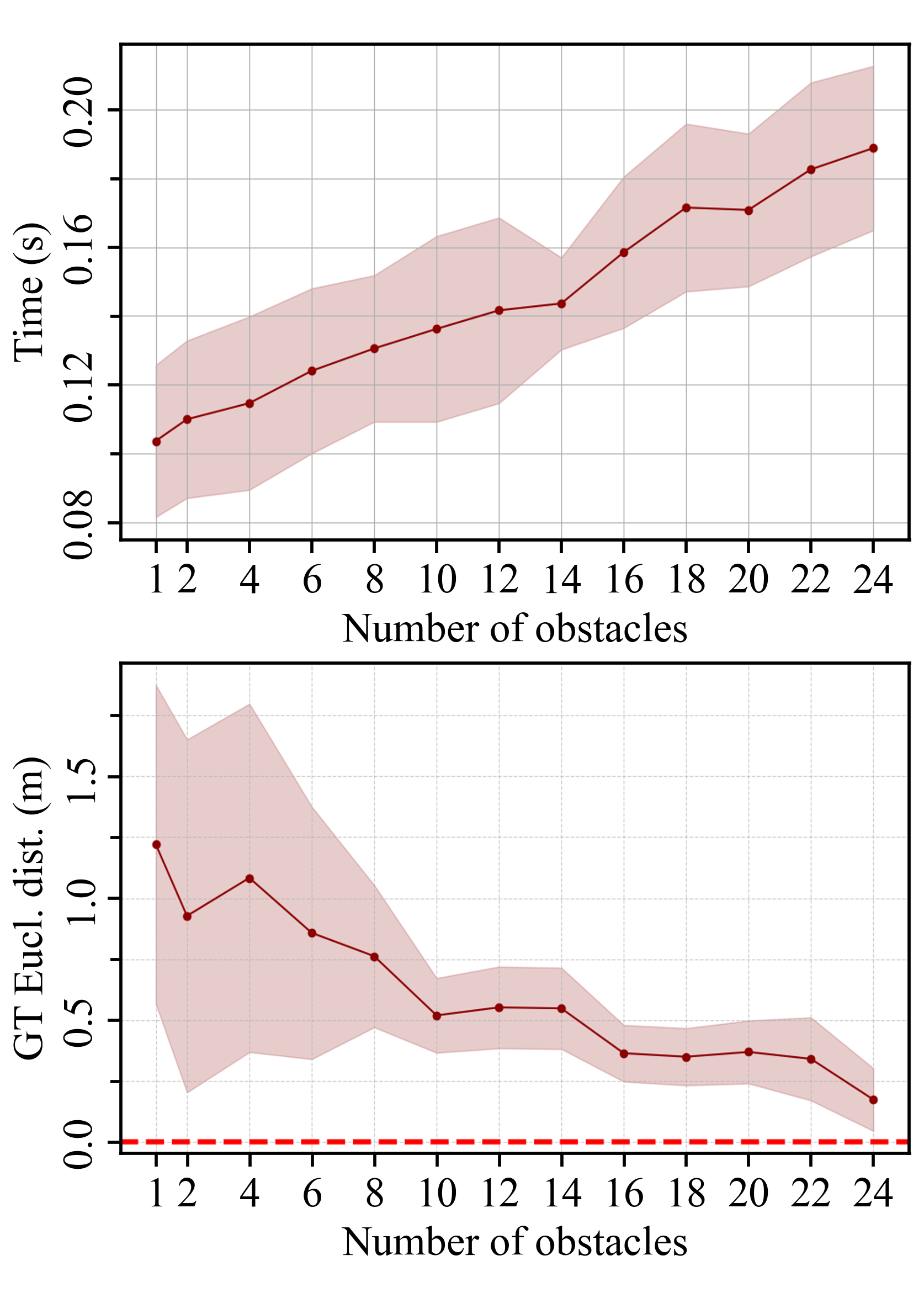}
    \caption{\textbf{Quantitative results.} Curves show means over five trials. \textbf{Top:} per-time-step computation time; shaded area denotes \(\pm 1\) standard deviation. \textbf{Bottom:} minimum ground-truth robot–obstacle distance; shaded area denotes \(\pm 1\) standard deviation.}
\label{fig:quantative_result}
\end{figure}

\section{Conclusion}
In this paper, we proposed a novel method to define Control Barrier Functions using 
Bernstein-Polynomial Signed Distance Fields (BP-SDFs).
This method utilizes Bernstein polynomials (BPs) to unify the representation of obstacles and robots.
The numerical results across diverse environments, robot dynamics, and homogeneous as well as heterogeneous teams verified the effectiveness of the proposed method. For future work, we will investigate updating the BP-SDF coefficients online using real-time observations, so that the geometry model can be refined after the initial offline learning stage when the shapes of the obstacles change over time.

\bibliographystyle{IEEEtran} 
\bibliography{ref}
\end{document}